\def\year{2020}\relax
\documentclass[letterpaper]{article} 
\usepackage{aaai20}  
\usepackage{times}  
\usepackage{helvet} 
\usepackage{courier}  
\usepackage[hyphens]{url}  
\usepackage{graphicx} 
\usepackage{graphicx}  
\newif\ifsup\supfalse
\suptrue

\usepackage{amsmath,amssymb,amsfonts}\allowdisplaybreaks
\usepackage{algorithm,algorithmic}
\usepackage[colorlinks=true,linkcolor=blue,urlcolor=blue,citecolor=blue]{hyperref}
\usepackage{dsfont}

\newcommand{\RR}{\mathbb{R}}

\newcommand{\cA}{\mathcal{A}}
\newcommand{\cB}{\mathcal{B}}
\newcommand{\cC}{\mathcal{C}}
\newcommand{\cD}{\mathcal{D}}
\newcommand{\cE}{\mathcal{E}}
\newcommand{\cF}{\mathcal{F}}

\newcommand{\cS}{\mathcal{S}}

\newcommand{\argmax}{\mathrm{argmax}}

\newtheorem{theorem}{Theorem}
\newtheorem{lemma}[theorem]{Lemma}
\newenvironment{proof}{\noindent {\textbf{Proof. }}}{$\Box$ \medskip}
\newtheorem{assumption}{Assumption}
\newenvironment{claim}[1]{\par\noindent\underline{Claim:}\space#1}{}
\newcommand{\abs}[1]{\left| #1 \right|}

\newcommand{\bOne}[1]{\mathds{1} \! \left\{#1\right\}}
\newcommand{\ip}[1]{\langle #1 \rangle}

\newcommand{\EE}[1]{\mathbb{E} \left[#1\right]}
\newcommand{\PP}[1]{\mathbb{P}\left[#1\right]}
\newcommand{\set}[1]{\left\{ #1 \right\}}

\newcommand{\opt}{\mathrm{opt}}
\newcommand{\feas}{\mathrm{feas}}
\newcommand{\val}{\mathrm{val}}
\newcommand{\LP}{\mathrm{LP}}
\newcommand{\NN}{\mathbb{N}}
\newcommand{\iin}{\mathrm{in}}

\newcommand{\KL}{\mathrm{KL}}
\usepackage{tikz}
\usepackage{mathtools}
\usetikzlibrary{arrows.meta}
\usetikzlibrary{calc}
\usepackage{subfigure}

\title{Stochastic Online Learning with Probabilistic Graph Feedback}
\author{Shuai Li,\textsuperscript{\rm 1} Wei Chen,\textsuperscript{\rm 2} Zheng Wen,\textsuperscript{\rm 3} Kwong-Sak Leung\textsuperscript{\rm 4} \\ 
\textsuperscript{\rm 1}Shanghai Jiao Tong University,  
\textsuperscript{\rm 2}Microsoft Research,
\textsuperscript{\rm 3}DeepMind, 
\textsuperscript{\rm 4}The Chinese University of Hong Kong\\ 
\textsuperscript{\rm 1}shuaili8@sjtu.edu.cn, \textsuperscript{\rm 2}weic@microsoft.com, \textsuperscript{\rm 3}zhengwen@google.com, \textsuperscript{\rm 4}ksleung@cse.cuhk.edu.hk
}
 
\begin{document}

\maketitle

\begin{abstract}
	We consider a problem of stochastic online learning with general probabilistic graph feedback, where each directed edge in the feedback graph has probability $p_{ij}$.
	Two cases are covered. (a) The one-step case, where after playing arm $i$ the learner observes a sample reward feedback of arm $j$ with independent probability $p_{ij}$.
	(b) The cascade case where after playing arm $i$ the learner observes feedback of all arms $j$ in a probabilistic cascade starting from $i$ -- for each $(i,j)$ with probability $p_{ij}$, if arm $i$ is played or observed, then a reward sample of arm $j$ would be observed with independent probability $p_{ij}$.
	Previous works mainly focus on deterministic graphs which corresponds to one-step case with $p_{ij} \in \{0,1\}$, an adversarial sequence of graphs with certain topology guarantees, or a specific type of random graphs.
	We analyze the asymptotic lower bounds and design algorithms in both cases. The regret upper bounds of the algorithms match the lower bounds with high probability.
\end{abstract}

\section{Introduction}

	Stochastic online learning is a general framework of sequential decision problem. At each time, the learner selects (or plays) an action from a given finite action set, receives some random reward and observes some random feedback. One simplest, though often unrealistic, feedback model is full-information feedback where the learning agent can observe the random rewards of all actions no matter which action is selected. Another popular feedback model is bandit feedback where only the random reward of the selected action is revealed to the learner \cite{auer2002finite}. Recent studies further generalize them to graph feedback where the feedback model is characterized by a (directed) graph \cite{mannor2011bandits}. Each edge $(i,j)$ means the learner will observe the random reward of action $j$ if playing action $i$. This problem is motivated by advertisements where the response for a vacation advertisement could provide side-information for a similar vacation place and social networks where the response from a user to a promotion could infer her neighbors to similar offers. 

	The problem of online learning with graph feedback has been extensively studied in both adversarial \cite{mannor2011bandits,alon2015online,kocak2014efficient,cohen2016online,kocak2016onlinenoisy} and stochastic settings \cite{caron2012leveraging,buccapatnam2014stochastic,tossou2017thompson,wu2015online}. While many of them assume self-loops on the feedback graphs, some succeed to remove this assumption \cite{alon2015online,wu2015online} where the reward of the selected action might be invisible. This general setting would fit into the partial monitoring framework \cite{bartok2014partial,komiyama2015regret}, but the literature on the latter mainly focus on finite case where the possible outcomes are finite. We also consider general feedback graphs that do not assume self-loops.

	Though some studies assume feedback graphs could vary over time or even invisible to the learner before selecting actions \cite{kocak2014efficient,tossou2017thompson}, most works focus on deterministic graphs or an adversarial list of graphs with certain topology guarantees. To the best of our knowledge, only a few of them work on probabilistic graphs with \cite{kocak2016online,alon2017nonstochastic} on adversarial case and \cite{liu2018information} on stochastic case and they only discuss about Erd\"os-R\'enyi random graphs \cite{erdHos1960evolution}. Recall that an Erd\"os-R\'enyi graph with parameter $p$ is by random sampling the edge of every pair of nodes with probability $p$ independently. 

	We consider general probabilistic feedback graphs in both the one-step case and the cascade case. The one-step case is the usual one where the learner observes reward of $j$ if edge $(i,j)$ exists in the random graph and $i$ is selected. The cascade case assumes the learner observes reward of $j$ if there is a (directed) path from $i$ to $j$ in the random graph and $i$ is selected. The observations of the cascade case, in other words, follow a probabilistic cascading starting from the selected action --- for each edge $(i,j)$ with probability $p_{ij}$, if action $i$ is either played or observed, then with an independent probability $p_{ij}$ a random reward sample of action $j$ will be observed. As a motivating example, consider the information propagation in social networks. If selecting a user in a social network causes an information cascade in the social network, one may be able to observe further feedback from the cascade users.

	This paper makes three major contributions.

	\begin{enumerate}
	\item We formalize the setting of stochastic online learning with general probabilistic graph feedback and consider both the one-step and the cascade cases.
	\item We derive asymptotic lower bounds for both the one-step and the cascade cases.
	\item We design algorithms for both the one-step and the cascade cases and analyze their regrets. Their asymptotic upper regret bounds match the asymptotic lower bounds with high probability.
	\end{enumerate}

	\paragraph{Related work}

	The studies on online learning with graph feedback started from adversarial online learning with side observations where a decision maker can observe rewards of other actions as well as observe the reward of the selected action \cite{mannor2011bandits}. The observation structure can be encoded as a directed graph where there is an edge $(i,j)$ if the reward of action $j$ is observed when $i$ is selected. Their setting assumes that self-loops exist on every node. Alon et al. \shortcite{alon2015online} then generalize to arbitrary directed graphs as long as each action is observable by selecting some action. They show the structure of feedback graph controls the inherent difficulty of the learning problem and present a classification over graphs. These works assume the feedback graph is fixed over time and known to the learner. A follow-up \cite{alon2015onlineArxiv} extends to time-varying feedback graphs where the graphs are revealed either at the beginning of the round or at the end of the round but assumes good topology properties on the graphs. Kocak et al. \shortcite{kocak2014efficient} also allow the feedback graph to vary over time and can be revealed to the learner at the end of the round. The results of \cite{kocak2016onlinenoisy} depend on the topological properties of the feedback graphs.
	Cohen et al. \shortcite{cohen2016online} assume the graph is not revealed in both adversarial and stochastic cases.
	All these works focus on the adversarial case.

	Besides \cite{cohen2016online}, there are also other works on the stochastic case with deterministic feedback graphs. Caron et al. \shortcite{caron2012leveraging} first study the stochastic case with side observations and design UCB-like algorithms with improved regret bound over the standard UCB without additional feedback. Buccapatnam et al. \shortcite{buccapatnam2014stochastic} derive an asymptotic lower bound and design two algorithms that are near-optimal. Tossou et al. \shortcite{tossou2017thompson} apply Thompson sampling and allow the feedback graph to be unknown and/or changing. They bound the Bayesian regret in terms of the size of minimum clique covering. 
	Wu et al. \shortcite{wu2015online} consider general feedback graphs but assume different observation variance from different choices of actions. They provide non-asymptotic problem-dependent regret lower bound and also design algorithms that achieve the problem-dependent lower bound and the minimax lower bounds. They are the first to remove the self-loop assumption in stochastic case.

	There are several works on specific Erd\"os-R\'enyi random feedback graphs where the feedback graph at each time is randomly generated by Erd\"os-R\'enyi model. Kocak et al. \shortcite{kocak2016online} consider adversarial case with the unknown generating probability of the feedback graphs. Liu et al. \shortcite{liu2018information} consider stochastic case and design a randomized policy with Bayesian regret guarantee. Also both of them assume self-observability. 
	An updated version \cite{alon2017nonstochastic} of Alon et al. \shortcite{alon2015onlineArxiv} extends one result to Erd\"os-R\'enyi model in the adversarial case. 
	We consider general probabilistic feedback graphs and provide gap-dependent regret bounds, which are also new in the setting of Erd\"os-R\'enyi random feedback graphs.

	The setting of graph feedback can be fit into a more general setting of partial monitoring \cite{rustichini1999minimizing,cesa2006prediction} where feedback matrix and reward matrix are given for each pair of the chosen action and the environment. Bartok et al. \shortcite{bartok2014partial} make a significant progress on classifying finite adversarial partial monitoring games which is completed by Lattimore and Szepesvari \shortcite{lattimore2019cleaning}. Komiyama et al. \shortcite{komiyama2015regret} derive a problem-dependent regret lower bound and design an algorithm with asymptotically optimal regret upper bound in the stochastic case. Most studies on general partial monitoring framework focus on finite case where the number of actions and possible outcomes are finite. The algorithms for general partial monitoring games are not efficient in our case since the feedback matrix might be infinite or exponentially large.

	The cascade observation feedback resembles the independent cascade model in the context of influence maximization studies \cite{kempe03,chen2013information}, but the goal is different: influence maximization aims at finding a set of $k$ seeds that generates the largest expected cascade size, while our goal is to find the best action (arm) utilizing the cascade feedback. Influence maximization has been combined with online learning in several studies \cite{vaswani2015influence,chen2016combinatorial,wen2017online,wang2017improving,saritacc2017combinatorial}, but again their goal is to maximize influence cascade size while using online learning to gradually learn edge probabilities.

\section{Settings}

	Our considered problem is characterized by a quadruple $(V, E, p, \mu)$, where $V=[K]$ is a set of $K$ actions, $E \subseteq V \times V$ is a set of directed edges between actions, $p: E \to (0, 1]$ maps edges to their \emph{triggering probabilities}, and $\mu = \{\mu_i \}_{i \in V}$ encodes the reward distributions of all actions. 
	The set of all possible reward distributions is denoted as $\cC$. Without loss of generality, we assume that each distribution candidate is $1$-sub-Gaussian.
	The set of all feasible vectors of reward distributions is denoted as $\cS$.
	The (directed) \emph{probabilistic feedback graph} is also denoted as $G = (V, E, p)$. We assume that the learner knows $G$ and the fact that $\mu_i$'s have $1$-sub-Gaussian tail, but does not know the reward mean $\theta_i$'s.

	At each time step $t=1,2, \ldots$, the environment first draws a reward vector $r_t = (r_t(i): i \in V)$ by independently sampling $r_t(i) \sim \mu_i$, and a random graph $G_t = (V, E_t)$ based on $G$.
	Specifically, $E_t = \{(i,j)\in E \, : \, o_{tij}=1 \} \subseteq E$, where $o_{tij}$ is an independent Bernoulli random variable with mean $p_{ij}$. Simultaneously, the learner adaptively chooses an action $i_t \in V$ based on its past observations, without observing $r_t$ or $G_t$. Then, the learner receives an instantaneous reward $r_t(i_t)$, and depending on the specific feedback model, it might also observe part of $r_t$.
	In this paper, we consider the following two feedback models:
	\paragraph{One-Step Triggering} The learner will receive feedback $(j, r_t(j))$ if and only if $(i_t, j) \in E_t$.
	\paragraph{Cascade Triggering} The learner will receive feedback $(j, r_t(j))$ if and only if there is a directed path from $i_t$ to $j$ in $G_t$.

	It is worth pointing out that though the learner receives the reward $r_t(i_t)$, however, if $(i_t, i_t)$ is not in $G_t$ in the one-step triggering case, or there is no directed circle from $i_t$ to $i_t$ in the cascade triggering case, $(i_t, r_t(i_t))$ is not observed. In other words, the learner might not observe the reward of its chosen action. Also note that existing works with graph feedback \cite{caron2012leveraging,buccapatnam2014stochastic,tossou2017thompson,alon2015online,wu2015online} are special cases of the one-step triggering case discussed above, with $p_{ij}=1$ for all $(i,j) \in E$.
	The work \cite{liu2018information} is also a special case of the one-step triggering case but with $p_{ij}$ having the same value.

	We assume the feedback graph is observable, that is each action has the chance to be observed by pulling some action.

	\begin{assumption}[observability]
	For each action $j$, there is an edge $(i,j) \in E$ for some $i$.
	\end{assumption}

	Next assumption states each feasible distribution vector is composed of distributions of ``same type''. For example, distributions over a bounded interval will not be put together with Gaussian distributions.

	\begin{assumption}[same type]
	\label{ass:same type}
	For each $\mu \in \cS$, $\KL(\mu_i, \mu_j)$ is well-defined for any $i,j\in V$. For each $\mu \in \cC^V$, if $\KL(\mu_i, \mu_j)$ is well-defined for any $i,j\in V$, then $\mu \in \cS$.
	\end{assumption}

	The last assumption says the $\KL$ divergence of the reward distributions is continuous with respect to the their means.

	\begin{assumption}[continuity]
	\label{ass:continuity}
	There exists some universal constant $B>0$ such that for each $\mu_i, \mu_j \in \cC$ and any $\epsilon \in (0,1)$, there exists $\mu_i' \in \cC$ satisfying $\KL(\mu_j, \mu_i')$ is well-defined, $\theta(\mu_i) + \epsilon \le \theta(\mu_i') < \theta(\mu_i) + 2\epsilon$ and $\abs{\KL(\mu_j, \mu_i') - \KL(\mu_j, \mu_i)} \le B\epsilon$.
	\end{assumption}

	The learner's objective is to maximize its expected cumulative reward, or equivalently, to minimize its expected cumulative regret
	\begin{align*}
	R_{\mu}(T; G) = T \max_{i \in V}\theta(\mu_i) - \EE{\textstyle \sum_{t=1}^T \theta(\mu_{i_t})}\,,
	\end{align*}
	where the expectation is over the randomness of $r_t$ and $G_t$. Here $\theta: \cC \to \RR$ is the mapping from the distributions to their means.

	We will omit $G$ in the regret expression and write $\theta(\mu_i)$ as $\theta_i$ if the context is clear. For simplicity, we assume there is only one best action and $\theta_1 > \theta_2 \ge \theta_3 \ge \cdots \ge \theta_K$. Denote $\theta = (\theta_i: i \in V)$.
	Let $\Delta_i(\mu) = \theta_1 - \theta_i$ be the reward gap between the best action and action $i$. Denote $\Delta(\mu) = (\Delta_i(\mu): i \in V)$. We will omit $\mu$ in the above notations if the context is clear.

	Let $V^{\iin}(j) = \{i \in [K]: (i,j) \in E\}$ be the set of incoming neighbors of action $j$.
	Let $N_i(t)$ be the number of times the learner selects an action $i$ and $N(t) = (N_i(t): i \in V)$ by the end of time $t$. 

	For general $\mu$, let $i_k(\mu)$ be the $k$-th best action index for the distributions $\mu$, which has the $k$-th largest mean. We will write $i_k$ for simplicity when the context is clear. Then $\theta_{i_1(\mu)} > \theta_i, \forall i\neq i_1(\mu)$. 

\section{Asymptotic Lower Bounds}

	\subsection{Lower Bound for One-Step Triggering}

		Define
		\begin{equation}
		\label{eq:define C(theta)}
		\begin{split}
		C(\mu) = \Bigg\{c\in [0,\infty)^V: \sum_{i \in V^\iin(1)}p_{i1} c_i \ge \frac{1}{\KL(\mu_2, \mu_1)} &\\
		\sum_{i\in V^\iin(j)} p_{ij} c_i \ge \frac{1}{\KL(\mu_j, \mu_1)}, ~\forall j\neq 1; &\Bigg\}\,.
		\end{split}
		\end{equation}
		Each element in the set represents an asymptotic pulling ``fraction'' of arms that can be used to distinguish these arms from the best arm.

		Recall that an algorithm is \textit{consistent} if $R_{\mu}(T) = o(T^a)$ for any $a>0$ and any feasible $\mu \in \cS$.
		Then the asymptotic lower bound for any consistent algorithm is provided in the following theorem.
		\begin{theorem}
		\label{thm:lower bound}
		For any consistent algorithm, the regret satisfies
		\begin{align}
		\label{eq:lower bound one-step}
		\liminf_{T \to \infty} \frac{R_{\mu}(T)}{\log T} \ge \inf_{c \in C(\mu)} \left\langle c, \Delta(\mu) \right\rangle\,.
		\end{align}
		\end{theorem}

		Note this lower bound can easily recover the lower bound in \cite[Theorem 3]{wang2017improving} where they only consider a special probabilistic graph $G$.

		\begin{proof}
			Fix any consistent algorithm and any distribution vector $\mu$.

			For any $j \neq 1$ and $n \ge 1$, by Assumption \ref{ass:continuity}, there exists a $\mu^{(n)}_j \in \cC$ such that $\theta_1 + \frac{1}{2^n} \le \theta\left(\mu^{(n)}_j\right) < \theta_1 + \frac{1}{2^{n-1}}$ and $\abs{\KL\left(\mu_j, \mu^{(n)}_j \right) - \KL(\mu_j, \mu_1)} \le \frac{B}{2^n}$. Define $\mu^{(n)} = \mu$ by setting $\mu^{(n)}_i = \mu_i$ for any $i \neq j$. Then by Assumption \ref{ass:same type}, $\mu^{(n)} \in \cS$.

			Let 
			\begin{align*}
			H = \{&i_1, \{r_1(j):(i_1,j) \in E_1\}; \\
			&i_2, \{r_2(j):(i_2,j) \in E_2\}; \ldots \}
			\end{align*}
			be the random variable of all outcomes, which is based on $\mu$, the algorithm and the graph realizations. Let $\mathbb{P}$ and $\mathbb{P}^{(n)}$ be the probability distribution over all possible realisations of outcomes when the distribution vector is $\mu$ and $\mu^{(n)}$ respectively.

			By high-dimensional Pinsker's inequality \cite[Lemma 5]{lattimore2017end},
			\begin{align*}
			\PP{N_{1}(T) < T/2} + \mathbb{P}^{(n)}[N_{1}(T) \ge T/2]& \\
			\ge \frac{1}{2} \exp\left( -\KL\left(\mathbb{P}, \mathbb{P}^{(n)}\right) \right)&\,.
			\end{align*}
			Note that
			\begin{align}
			\label{eq:kl in lower bound}
			\KL\left(\mathbb{P},\mathbb{P}^{(n)}\right) &=  \sum_{i\in V^{\iin}(j)}p_{ij}\EE{N_i(T)} \KL\left(\mu_j, \mu_j^{(n)}\right) \\
			&\le \KL\left(\mu_j, \mu_j^{(n)}\right) \sum_{i\in V^{\iin}(j)}p_{ij}\EE{N_i(T)}\,. \notag
			\end{align}
			Then 
			\begin{align*}
			&\sum_{i\in V^{\iin}(j)}p_{ij}\EE{N_i(T)} \\
			\ge& \frac{1}{\KL\left(\mu_j, \mu_j^{(n)}\right)} \\
			&\qquad \cdot \log\frac{1/2}{\PP{N_{1}(T) < T/2} + \mathbb{P}^{(n)}[N_{1}(T) \ge T/2]}\\
			\ge&  \frac{1}{\KL\left(\mu_j, \mu_j^{(n)}\right)} \\
			&\qquad \cdot \log\frac{1/2}{R_{\mu}(T)/(\Delta_{2} \cdot T/2) + R_{\mu^{(n)}}(T) / \left(\frac{1}{2^n} \cdot T/2\right)}\\
			=&  \frac{1}{\KL\left(\mu_j, \mu_j^{(n)}\right)} \log\frac{T/4}{R_{\mu}(T)/\Delta_{2} + R_{\mu^{(n)}}(T) / \frac{1}{2^n}} \,,
			\end{align*}
			where the second inequality is due to 
			\begin{align*}
				&R_{\mu}(T) \ge \PP{N_{1}(T) < T/2} \Delta_2 \cdot T/2 \,,\\
				&R_{\mu^\epsilon}(T) \ge \PP{N_{1}(T) \ge T/2} \left(\theta\left(\mu^{(n)}_j \right) - \theta_1\right) T/2\,.
			\end{align*}

			Since the algorithm is consistent, $R_{\mu}(T) = o(T^a)$ and $R_{\mu^{(n)}}(T) = o(T^a)$ for any $a>0$, or equivalently 
			\begin{align*}
			\limsup_{T\to \infty} \frac{\log{R_{\mu}(T)}}{\log{T}} = 0\,,\quad \limsup_{T\to \infty} \frac{\log{R_{\mu^{(n)}}(T)}}{\log{T}} = 0\,.
			\end{align*}
			Thus
			\begin{align*}
			\sum_{i\in V^{\iin}(j)} p_{ij} \liminf_{T\to\infty} \frac{\EE{N_i(T)}}{\log T} \ge \frac{1}{\KL\left(\mu_j, \mu_j^{(n)}\right)} \,.
			\end{align*}
			Next take $n \to \infty$,
			\begin{align*}
			\sum_{i \in V^{\iin}(j)} p_{ij} \liminf_{T\to\infty} \frac{\EE{N_i(T)}}{\log T} \ge \frac{1}{\KL(\mu_j, \mu_1)} \,.
			\end{align*}

			For $j=1$ and $n\ge 1$, take $\mu^{(n)}=\mu$ except $\mu^{(n)}_2 \neq \mu_2$ with $\theta_1 + \frac{1}{2^n} \le \theta\left(\mu^{(n)}_2 \right) < \theta_1 + \frac{1}{2^{n-1}}$ and $\abs{\KL\left(\mu_2, \mu^{(n)}_2 \right) - \KL(\mu_2, \mu_1)} \le \frac{B}{2^n}$. Similar result follows
			\begin{align*}
			\sum_{i \in V^{\iin}(1)} p_{i1} \liminf_{T\to\infty} \frac{\EE{N_i(T)}}{\log T} \ge \frac{1}{\KL(\mu_2, \mu_1)} \,.
			\end{align*}

			Thus the vector $\liminf_{T\to\infty} \frac{\EE{N(T)}}{\log T} \in C(\mu)$. Recall the regret is $R_{\mu}(T) = \sum_{i=1}^K \EE{N_i(T)}\Delta_i(\mu)$. The result follows.
		\end{proof}

	\subsection{Lower Bound for Cascade Triggering}

		Let $p_{ij}'$ be the probability that there is a directed path from $i$ to $j$ in a random realization of $G$.
		Define
		\begin{align*}
		C'(\mu) = \Bigg\{c\in [0,\infty)^V: \sum_{i} p_{i1}' c_i \ge \frac{1}{\KL(\mu_2, \mu_1)} &\\
		\sum_{i} p_{ij}' c_i \ge \frac{1}{\KL(\mu_j, \mu_1)}, ~\forall j\neq 1&\Bigg\}\,.
		\end{align*}
		\begin{theorem}
		\label{thm:lower bound cascade}
		For any consistent algorithm, the regret satisfies
		\begin{align*}
		\liminf_{T \to \infty} \frac{R_{\mu}(T)}{\log T} \ge \inf_{c \in C'(\mu)} \langle c, \Delta(\mu) \rangle\,.
		\end{align*}
		\end{theorem}

		This proof is similar to the above one by replacing \eqref{eq:kl in lower bound} with the following formula
		\begin{align*}
		\KL\left(\mathbb{P},\mathbb{P}^{(n)}\right) = \sum_{i} p_{ij}'\EE{N_{i}(T)} \KL\left(\mu_j, \mu_j^{(n)} \right) \,.
		\end{align*}

		Note that the computation of $p_{ij}'$ is \#P-hard for general graphs \cite{Valiant79,wang2012scalable}. 
		Thus the lower bound is not efficiently computable even when $\mu$ is known.

\section{Algorithm and Analysis}

	In this section, we design algorithms that can match the lower bounds with high probability asymptotically.
	The lower bounds in the last section are stated in terms of $\KL$-divergence of distributions. Since the $\KL$-divergence of a real distribution and its estimated empirical distribution might be undefined, we assume the $\KL$-divergence of distributions could be represented by their corresponding means and is also continuous in means, which is also a tradition in bandit area. For example, a previous work \cite{wu2015online} assumes distributions to be Gaussian to make statement simpler. We will give more discussions in Section \ref{sec:discussions}. In the following, we use mean vector $\theta$ to represent the vector of distributions $\mu$ for simplicity.

	Let $\hat{\theta}_t$ be the sample-mean estimates of $\theta$ by the end of time $t$.
	Let $n_{ij}(t)$ be the number of times that action $i$ is selected and reward for action $j$ is observed by the end of time $t$. Then $\EE{n_{ij}(t) \mid N_i(t)} = N_i(t) p_{ij}$.
	Let $m_j(t) = \sum_{i} n_{ij}(t)$ be the number of observations for action $j$ by the end of time $t$.

	\subsection{One-Step Uniform Case}

		The uniform case in which all $p_{ij}$'s have the same value $p$ is first considered in this section. When $E$ contains edges between every pair of actions, this graph reduces to Erd\"os-R\'enyi random graph with parameter $p$.

		Let $M_j(t) = \sum_{i \in V^{\iin}(j)} N_i(t) p$ be the expected number of observations for action $j$ at the end of time $t$. Then $\EE{m_j(t) \mid M_j(t)} = M_j(t)$.

		\begin{algorithm}[thb!]
		\caption{One-Step Uniform Case}
		\label{algo:uniform case}
		\begin{algorithmic}[1]
		\STATE \label{alg:uniform:initialization} Set $N^e(0) = 0$ and $\hat{\theta}_0 = (1,1,...,1)$.
		\FOR{$t=1,2,\ldots$}
		\IF{\label{alg:uniform:if explore p} $m_j(t-1) < M_j(t-1) / 2$ for some $j$ 
		}
		\STATE \label{alg:uniform:explore p} Play $i_t \in V^{\iin}(j)$;
		\STATE \label{alg:uniform:explore p, not increase Ne} $N^e(t) = N^e(t-1)$;
		\ELSIF{\label{alg:uniform:if good condition}$\frac{N(t-1)}{16\log(t-1)} \in C(\hat{\theta}_{t-1})$
		}
		\STATE \label{alg:uniform:exploitation} Play $i_t = i_1(\hat{\theta}_{t-1})$; 
		\STATE \label{alg:uniform:not increase Ne} $N^e(t) = N^e(t-1)$
		\ELSIF{\label{alg:uniform:if accurate theta} $M_j(t-1) < 2 \beta\left(N^e(t-1)\right) / K $ for some $j$}
		\STATE \label{alg:uniform:accurate theta} Play $i_t \in V^{\iin}(j)$;
		\STATE \label{alg:uniform:accurate theta, increase Ne} $N^e(t) = N^e(t-1) + 1$;
		\ELSE
		\STATE \label{alg:uniform:explore C} Play $i_t$ such that $N_i(t-1) < 16 \ c_i(\hat{\theta}_{t-1}) \log(t-1)$;
		\STATE \label{alg:uniform:explore C, increase Ne} $N^e(t) = N^e(t-1) + 1$;
		\ENDIF
		\ENDFOR
		\end{algorithmic}
		\end{algorithm}
		The pseudocode of the algorithm is provided in Algorithm \ref{algo:uniform case}. It starts with the initialization of $N^e$ and the estimates of $\theta$ (line \ref{alg:uniform:initialization}). Here $N^e$ is the number of exploration rounds for the learner to know more about unknown $\theta$ which will be clearer later. At each time $t$, if for some $j$ the real observation times of action $j$ is less than half the expected observation times (line \ref{alg:uniform:if explore p}), then the learner selects a parent of $j$ to try to observe reward of $j$ once more (line \ref{alg:uniform:explore p}) and keeps $N^e$ unchanged (line \ref{alg:uniform:explore p, not increase Ne}). Note that $\EE{m_j(t) \mid M_j(t) = m} = m$ and $m_j(t)$ will concentrate at $m$ as $m$ goes to infinity. The condition $m_j(t) < M_j(t) / 2$ means part of the realizations of graph $G$ is far from the expectation and $2$ can be changed to other larger-than-$1$ constant. This is one of the key differences from deterministic graph feedback \cite{wu2015online} where the number of observations is well controlled by just selecting actions. While under the probabilistic graph feedback, there is a gap between the number of real observations and expected number of observations.

		When $m_j(t) \ge M_j(t) / 2$ for all $j$, then the realizations of $G$ are good enough and the learner can rely on the quantities of selections to control the accuracy of the estimates. If the selection vector is good enough for current $\hat{\theta}$ under current accuracy level (line \ref{alg:uniform:if good condition}), then the learner will exploit the current best action (line \ref{alg:uniform:exploitation}) and keep $N^e$ unchanged. Here $C(\cdot)$ is defined as in \eqref{eq:define C(theta)} and represents the set of good selected ``fractions'' of actions that are able to identify the reward gaps between actions.

		If the current selection vector $N$ is not good enough, then the learner will first check if $\hat{\theta}$ is close enough to $\theta$ (line \ref{alg:uniform:if accurate theta}-\ref{alg:uniform:accurate theta, increase Ne}) and if yes, will explore according to current $\hat{\theta}$. The number $N^e$ of exploration rounds for the learner to know more about $\theta$ will increase in this part (line \ref{alg:uniform:accurate theta, increase Ne}\&\ref{alg:uniform:explore C, increase Ne}). The condition of line \ref{alg:uniform:if accurate theta} has an auxiliary function $\beta:\NN \to [0,\infty)$ to guide the exploration such that $\hat{\theta}$ will be close to $\theta$ in the long run. This auxiliary function is also crucial in previous work \cite{wu2015online} to control the regret bound in the asymptotic sense. The auxiliary function $\beta$ can be any non-decreasing function satisfying $0 \le \beta(n) \le n / 2$ and the subadditivity $\beta(m+n) \le \beta(m) + \beta(n)$.
		If some component of $\hat{\theta}$ has not been explored enough (line \ref{alg:uniform:if accurate theta}), then the learner selects a parent to try to get one more observation (line \ref{alg:uniform:accurate theta}) and increases $N^e$ (line \ref{alg:uniform:accurate theta, increase Ne}).

		When all components of $\hat{\theta}$ are close to $\theta$, the learner selects an action according to the current $\hat{\theta}$ with minimal cost on the regret instructed by the asymptotic lower bound \eqref{eq:lower bound one-step}. Here $c_i(\theta')$ denotes any optimal solution of the linear programming problem that minimizes $\ip{c, \theta'}$ among all $c \in C(\theta')$. 
		Since $\hat{\theta}$ is close enough to $\theta$ under current accuracy level, the vector $c_i(\hat{\theta}_{t-1})$ is close enough to $c_i(\theta)$ (which is part of the proof for the following theorem). There must be at least an $i$ such that $N_i(t-1) < 16 \ c_i(\hat{\theta}_{t-1}) \log(t-1)$ or else the condition of line \ref{alg:uniform:if good condition} holds.

		The regret bound for the algorithm is stated as follows.

		\begin{theorem}
		\label{thm:upper bound:uniform case}
		The regret of Algorithm \ref{algo:uniform case} for one-step uniform case satisfies for any $\epsilon>0$, 
		\begin{equation}
		\begin{split}
		R_\theta(T) &\le 4 \log(T) \sum_{i=1}^K c_i(\theta, \epsilon) \Delta_i(\theta) \\
		&+ 10 \log(KT^2) \sum_{i=1}^K \frac{\Delta_i(\theta)}{p} + 4 \sum_{s=0}^T \exp\left(-\frac{\beta(s)\epsilon^2}{2K}\right)  \\
		& + 2 \beta\left(4 \sum_{i=1}^K c_i(\theta, \epsilon) \log(T) + K\right) + 15 K\,,
		\end{split}
		\end{equation}
		where $c_i(\theta, \epsilon) = \sup \{c_i(\theta'): \abs{\theta_j'-\theta_j} \le \epsilon, ~~\forall j \in [K]\}$.

		Assume $\beta(n)=o(n)$ and $\sum_{s=0}^{\infty} \exp\left(-\frac{\beta(s) \epsilon^2}{2K}\right) < \infty$ for any $\epsilon>0$. Then for any $\theta$ such that $c(\theta)$ is unique, 
		\begin{align}
		\limsup_{T \to \infty} R_\theta(T) / \log(T) \le 4 \inf_{c \in C(\theta)} \ip{ c, \Delta(\theta) }
		\end{align}
		holds with probability at least $1-\delta$ for any $\delta>0$.
		\end{theorem}
		Note that any $\beta(n) = a n^b$ with $a \in \left(0, \frac{1}{2}\right], b\in (0,1)$ meets the requirements. 
		The proof is by bounding the forced exploration (line \ref{alg:uniform:if accurate theta}-\ref{alg:uniform:accurate theta, increase Ne}), the exploration by LP solutions (line \ref{alg:uniform:explore C}-\ref{alg:uniform:explore C, increase Ne}) and the exploitation (line \ref{alg:uniform:if good condition}-\ref{alg:uniform:not increase Ne}). The main difference with previous works is to bound the difference of realized random graphs and the expected graph (line \ref{alg:uniform:if explore p}-\ref{alg:uniform:explore p, not increase Ne}). The detailed proof is provided in
		\ifsup
		Section \ref{sec:proofs of upper bound one step}.
		\else
		\cite{li2019stochastic}. 
		\fi

	\subsection{One-Step General Case}

		In the general case where $p_{ij}$ can be different, $M_j(t) = \sum_{i \in V^{\iin}(j)} N_i(t) p_{ij}$. The algorithm follows as in Algorithm \ref{algo:uniform case} by only replacing line \ref{alg:uniform:explore p} with 
		\begin{itemize}
		\item[(\ref{alg:uniform:explore p}')] Play $i_t \in \argmax_{i \in V^\iin(j)} p_{ij}$.
		\end{itemize}

		Let
		\begin{align}
		V^e = \set{i \in [K]: i \in \argmax_{i' \in V^\iin(j)} p_{i'j} \text{ for some }j} \label{eq:V e}
		\end{align}
		be the set of exploration nodes that have the largest live probability among all incoming edges to some $j$. Let
		\begin{align}
		p_{i}^e = \min \set{p_{ij}: i \in \argmax_{i' \in V^\iin(j)} p_{i'j} \text{ for some }j} \label{eq:p i e}
		\end{align}
		be the minimal exploration probability for any $i \in V^e$.
		With a modified proof to the uniform case, the theoretical guarantee for the general case follows.
		\begin{theorem}
		\label{thm:upper bound:general case}
		The regret of the modified Algorithm \ref{algo:uniform case}' for one-step general case satisfies for any $\epsilon>0$, 
		\begin{equation}
		\begin{split}
		R_\theta(T) &\le 4 \log(T) \sum_{i=1}^K c_i(\theta, \epsilon) \Delta_i(\theta) \\
		&+ 10 \log(KT^2) \sum_{i\in V^e} \frac{\Delta_i(\theta)}{p_{i}^e} \\
		&+ 4 \sum_{s=0}^T \exp\left(-\frac{\beta(s)\epsilon^2}{2K}\right) \\
		&+ 2 \beta\left(4 \sum_{i=1}^K c_i(\theta, \epsilon) \log(T)\right) + 15 K\,.
		\end{split}
		\end{equation}

		Assume $\beta(n)=o(n)$ and $\sum_{s=0}^{\infty} \exp\left(-\frac{\beta(s) \epsilon^2}{2K}\right) < \infty$ for any $\epsilon>0$. Then for any $\theta$ such that $c(\theta)$ is unique, 
		\begin{align}
		\limsup_{T \to \infty} R_\theta(T) / \log(T) \le 4 \inf_{c \in C(\theta)} \ip{c, \Delta(\theta)}
		\end{align}
		holds with probability at least $1-\delta$ for any $\delta>0$.
		\end{theorem}

	\subsection{Cascade Case}

		\begin{algorithm}[thb!]
		\caption{Cascade Case}
		\label{algo:cascade case}
		\begin{algorithmic}[1]
		\STATE Set $N^e(0) = 0$ and $\hat{\theta}_0 = (1,1,...,1)$. $\eta:\NN_+ \to [0,1)$.
		\FOR{$t=1,2,\ldots$}
		\IF{\label{alg:cascade:if explore p} $m_{j}(t-1) < M_j'(t-1) \ /2$ for some $j$
		}
		\STATE \label{alg:cascade:explore p} Play $i_t=i$ if $(P_t)_{ij} \ge \frac{1}{2} \max_{i'} (P_t)_{i'j}$;
		\STATE \label{alg:cascade:explore p, not increase Ne} $N^e(t) = N^e(t-1)$;
		\ELSIF{\label{alg:cascade:if good condition}$\frac{N(t-1)}{16\log(t-1)} \in S_{\feas}(\LP_t)$
		}
		\STATE \label{alg:cascade:exploitation} Play $i_t = i_1(\hat{\theta}_{t-1})$;
		\STATE \label{alg:cascade:not increase Ne} $N^e(t) = N^e(t-1)$
		\ELSIF{\label{alg:cascade:if accurate theta} $M_j'(t-1) < 2 \beta(N^e(t-1)) / K$ for some $j$}
		\STATE \label{alg:cascade:accurate theta} Play $i_t \in V^{\iin}(j)$;
		\STATE \label{alg:cascade:accurate theta, increase Ne} $N^e(t) = N^e(t-1) + 1$;
		\ELSE
		\STATE \label{alg:cascade:explore C} Play $i_t=i$ such that $N_i(t-1) < 16 \ c_{t, i} \log(t-1)$ where $c_t \in S_{\opt}(\LP_t)$;
		\STATE \label{alg:cascade:explore C, increase Ne} $N^e(t) = N^e(t-1) + 1$;
		\ENDIF
		\ENDFOR
		\end{algorithmic}
		\end{algorithm}

		For the deterministic graphs, there is no essential difference between one-step case and cascade case --- the cascade case on a deterministic graph would be equivalent to constructing a new graph where an edge exists if and only if there is a path on the original graph. 
		For a probabilistic graph, one might try a similar solution for the cascade case by constructing a new graph $G'$ where the probability of an edge $(i,j)$ is just the  probability $p_{ij}'$ of $i$ connecting to $j$ in a random realization of the original graph.
		However the computation of $p_{ij}'$ is \#P-hard for general graphs, and thus the accurate graph $G'$ is unattainable, though it can be approximated within any accuracy by Monte Carlo simulations. 
		Therefore, during the running of the algorithm, a reasonable approximation of $G'$ is needed.

		Define ${V^e}'$ and ${p_i^e}'$ similarly with \eqref{eq:V e} and \eqref{eq:p i e} by replacing $p_{ij}$ with $p_{ij}'$. Since the computation of $p_{ij}'$ is \#P-hard, we define an estimated version of ${V^e}'$ and ${p_i^e}'$ respectively:
		\begin{align*}
		&\hat{V}^e = \set{i \in [K]: p_{ij}' \ge \frac{1}{2} \max_{i'} p_{i'j}' \text{ for some }j}\\
		&\hat{p}_i^e = \min \set{p_{ij}': p_{ij}' \ge \frac{1}{2} \max_{i'} p_{i'j}' \text{ for some }j}
		\end{align*}
		for any $i \in \hat{V}^e$. Then $\hat{p}_i^e \ge {p_{i'}^e}' / 2$ for some $i'$.

		To overcome the stated challenge, we need an auxiliary functions $\eta : \NN_+ \to [0,1)$ to set up the tolerance of the approximation. At each time $t$, the path from $i$ to $j$ with probability $p_{ij}' \le \eta(t)$ can be treated as nonexistent (with probability $0$) and the estimation of $p_{ij}'$ has noise at most $\eta(t)/2$ if the real value $p_{ij}' > \eta(t)$. Any non-increasing function with limit $0$ can be chosen as $\eta$. 
		The choice of $\eta$ is to control the complexity of the graph with only focusing the path of a reasonable length.

		Let $\LP(\theta', \eta)$ be the following linear programming problem
		\begin{equation}
		\label{eq:lp}
		\begin{split}
		&\min \ip{\Delta(\theta'), c} \\
		&\text{over all } c \in \RR^K  \text{ satisfying } P^\top c \ge b(\theta')  \text{ and } c \ge 0
		\end{split}
		\end{equation}
		where $P \in [0,1]^{K \times K}$ satisfies $P_{ij} = 0$ if $p_{ij}' \le \eta$ and $\abs{P_{ij} - p_{ij}'} \le \eta/2$ if $p_{ij}' > \eta$ and $b_i(\theta') = \frac{1}{\Delta_i^2(\theta')}$ for $i \neq i_1(\theta')$ and $b_{i_1(\theta')}(\theta') = \frac{1}{ \Delta_{i_2(\theta')}^2(\theta')}$. 

		With the approximation $G_t$ and the estimated value for reward vector $\hat{\theta}_{t-1}$, the linear programming problem considered in time $t$ is $\LP_t = \LP(\hat{\theta}_{t-1}, \eta(t))$ and the corresponding $P$ in \eqref{eq:lp} is denoted as $P_t$. Then the algorithm runs with $\LP_t$ accordingly. The complete pseudocode is presented in Algorithm \ref{algo:cascade case}. In particular, the examination on the realization is performed on approximated graph $G_t$ with probability matrix $P_t$ (line \ref{alg:cascade:if explore p}). The exploitation condition is on the $\LP_t$ (line \ref{alg:cascade:if good condition}). Here $S_{\feas}(\LP_t)$ is the feasible solution set of the linear programming problem $\LP_t$ which is the set of all $c \in \RR^K$ satisfying $P_t^\top c \ge b, c \ge 0$. The exploration when all components of estimated $\hat{\theta}$ are accurate enough with minimal cost instructed by linear programming solutions is also related to $\LP_t$ (line \ref{alg:cascade:explore C}). Here $S_{\opt}(\LP_t)$ is the optimal solution set of $\LP_t$.

		Also $M_j'(t) = \sum_{i} N_i(t) (P_t)_{ij}$ is changed accordingly.

		The regret of the Algorithm \ref{algo:cascade case} is upper bounded in the following theorem.
		\begin{theorem}
		\label{thm:upper bound:cascade case}
		The regret of the Algorithm \ref{algo:cascade case} for cascade case satisfies for any $\epsilon>0$,
		\begin{equation}
		\begin{split}
		R_\theta(T) \le &4 \sum_{i=1}^K \Delta_i(\theta) \max_{t\in [T]} \set{c_i(\theta, \epsilon, \eta(t)) \log(t)} \\
		&+  10 \log(KT^2) \sum_{i \in \hat{V}^e} \frac{\Delta_i(\theta)}{\hat{p}_{i}^e} \\
		&+ 2 \beta\left(4 \sum_{i=1}^K \max_{t\in [T]} \set{c_i(\theta, \epsilon, \eta(t)) \log(t)} + K\right) \\
		&+ 4 \sum_{s=0}^T \exp\left(-\frac{\beta(s)\epsilon^2}{2K}\right) + 15K \,,\\
		\end{split}
		\end{equation}
		where
		\begin{align*}
		c_i(\theta, \epsilon, \eta) = \\
		\sup \big\{c_i: \ &c \in S_{\opt}(\LP(\theta', \eta)) \text{ and } \abs{\theta_j'-\theta_j} \le \epsilon, ~~\forall j \in [K]\big\}\,.
		\end{align*}

		Assume $\beta(n)=o(n)$ and $\sum_{s=0}^{\infty} \exp\left(-\frac{\beta(s) \epsilon^2}{2K}\right) < \infty$ for any $\epsilon>0$. Then for any $\theta$ such that $c(\theta)$ is unique, 
		\begin{align}
		\limsup_{T \to \infty} R_\theta(T) / \log(T) \le 4 \inf_{c \in C'(\theta)} \ip{ c, \Delta(\theta) }
		\end{align}
		holds with probability at least $1-\delta$ for any $\delta>0$.
		\end{theorem}

		The result depends on the robustness of the linear programming problems. The $P$ matrix in the LP problem \eqref{eq:lp} is noisy, which is much different from one-step case and the case of deterministic graphs where the noise is only on $\theta'$. See discussions in the next section. The full proof is put in 
		\ifsup
		Section \ref{sec:proof of cascade case}.
		\else
		\cite{li2019stochastic}. 
		\fi

	\subsection{Discussions}
	\label{sec:discussions}

		The assumptions on the reward distributions are mainly used to ensure that the learning algorithms are able to differentiate them in the worst case (or the regret lower bound). The Gaussian distribution, Bernoulli distribution and common continuous random distribution on a common bounded interval like Beta distribution all satisfy the requirements.

		The assumption that the reward distribution can be represented by its mean is commonly adopted in bandit literature. Since there is always gap between a continuous distribution with its discrete empirical estimate and the reward only cares about the mean, previous works hardly choose to estimate the real distribution but mainly choose to estimate the mean. The real mean can be well analysed by constructing a confidence interval around the sample mean.


		The term $O\left(\log(T) \sum_{i=1}^K \frac{\Delta_i(\theta)}{p}\right)$ in the regret bound for one-step uniform case (same for other two cases) is due to the gap between the realizations and the expectations of the probabilistic graphs. Such a term can be removed in the asymptotic sense with high probability based on a different proof. With high probability, the connection between the realizations and the expectations of the probabilistic graphs can be guaranteed for large enough $T$, so the realizations of the probabilistic graphs are good enough and no regret would be caused from line \ref{alg:uniform:if explore p} - \ref{alg:uniform:explore p, not increase Ne} of Algorithm \ref{algo:uniform case} for large enough $T$. If we remove the high probability condition, such a $1/p$ term remains in the asymptotic sense.
		Such $1/p$ term also appears in the regret $O(\sqrt{T/p})$ of \cite{kocak2016online} on Erd\"os-R\'enyi random graphs in adversarial setting, as compared with adversarial case on deterministic graphs. 
		It is not clear whether this $1/p$ term represents hindsight difficulty between the probabilistic graphs and deterministic graphs. This would be an interesting future direction.

		The terms $\{p_i^e: i \in [K]\}$ in the one-step general case describes the minimal exploration probabilities to observe every action. For each $i \in [K]$, $p_i^e = \max_{i'} p_{i'j}$ for some $j$, that is $p_{ij}$ is the largest live probabilities among all incoming edges for some $j$. These terms represent the problem complexities for the underlying probabilistic graph. When all $p_{ij}$ are equal to $p$, $p_i^e = p$.

		The term ${p_i^e}'$ in the cascade case is usually larger than $p_i^e$ since it takes the same operations on the connection probabilities of incoming paths which are larger than live probabilities of incoming edges. The term $\hat{p}_i^e$ is an estimation satisfying $\hat{p}_i^e \ge {p_{i'}^e}' / 2$ for some $i'$.

		Next we discuss the difference in proof of the cascade case. If the noise of the linear programming problems is on the $b$ vector in \eqref{eq:lp}, then by the standard results in statistics \cite[\S3C.5]{dontchev2009implicit}, the resulting optimal solution sets are Lipschitz continuous. The property of Lipschitz continuity is essential since actions are selected according to the optimal solution of a noisy LP problem (line \ref{alg:cascade:explore C}) and we need to guarantee this kind of selections is safe. The noise on $\Delta$ vector in \eqref{eq:lp} is also easy to deal with by considering the dual problem. However, it is much different if the noise is on the $P$ matrix. For example, consider the LP problem that minimizes $x$ over all $a x \ge 1$ and $x \ge 0$ with parameter $a>0$. The optimal solution $x^\ast = 1/a$ is not Lipschitz continuous with respect to $a$. So the standard statistical tools could not apply here. We derive a novel property of the Lipschitz continuity when there is noise on $P$ for our specific $P$ matrix. 

		Last we would like to stress that our regret bounds are the first gap-dependent bounds even under the one-step uniform case, which contains the simple case of Erd\"os-R\'enyi random graph feedback. The previous works on Erd\"os-R\'enyi random graphs study gap free bound, no matter in the stochastic setting or the adversarial setting.

\section{Conclusion and Future Work}

	We are the first to formalize the setting of stochastic online learning with probabilistic feedback graph. We derive asymptotic lower bounds for both one-step and cascade cases. The regret bounds of our designed algorithms match the lower bounds with high probability. 

	This framework is new and we only provide asymptotic lower bounds and finite-time problem-dependent upper bounds. Finite-time lower bounds and minimax upper/lower bounds are all interesting future directions. Deriving Bayesian regret bounds is also an interesting topic.

\section*{Acknowledgement}

	Thank Houshuang Chen for help on the experiments.

\bibliographystyle{aaai}
\bibliography{ref}

\begin{thebibliography}{}

\bibitem[\protect\citeauthoryear{Alon \bgroup et al\mbox.\egroup
  }{2015a}]{alon2015online}
Alon, N.; Cesa-Bianchi, N.; Dekel, O.; and Koren, T.
\newblock 2015a.
\newblock Online learning with feedback graphs: Beyond bandits.
\newblock In {\em Conference on Learning Theory},  23--35.

\bibitem[\protect\citeauthoryear{Alon \bgroup et al\mbox.\egroup
  }{2015b}]{alon2015onlineArxiv}
Alon, N.; Cesa-Bianchi, N.; Dekel, O.; and Koren, T.
\newblock 2015b.
\newblock Online learning with feedback graphs: Beyond bandits.
\newblock {\em arXiv preprint arXiv:1502.07617}.

\bibitem[\protect\citeauthoryear{Alon \bgroup et al\mbox.\egroup
  }{2017}]{alon2017nonstochastic}
Alon, N.; Cesa-Bianchi, N.; Gentile, C.; Mannor, S.; Mansour, Y.; and Shamir,
  O.
\newblock 2017.
\newblock Nonstochastic multi-armed bandits with graph-structured feedback.
\newblock {\em SIAM Journal on Computing} 46(6):1785--1826.

\bibitem[\protect\citeauthoryear{Auer, Cesa-Bianchi, and
  Fischer}{2002}]{auer2002finite}
Auer, P.; Cesa-Bianchi, N.; and Fischer, P.
\newblock 2002.
\newblock Finite-time analysis of the multiarmed bandit problem.
\newblock {\em Machine Learning} 47(2-3):235--256.

\bibitem[\protect\citeauthoryear{Bart{\'o}k \bgroup et al\mbox.\egroup
  }{2014}]{bartok2014partial}
Bart{\'o}k, G.; Foster, D.~P.; P{\'a}l, D.; Rakhlin, A.; and Szepesv{\'a}ri, C.
\newblock 2014.
\newblock Partial monitoring—classification, regret bounds, and algorithms.
\newblock {\em Mathematics of Operations Research} 39(4):967--997.

\bibitem[\protect\citeauthoryear{Buccapatnam, Eryilmaz, and
  Shroff}{2014}]{buccapatnam2014stochastic}
Buccapatnam, S.; Eryilmaz, A.; and Shroff, N.~B.
\newblock 2014.
\newblock Stochastic bandits with side observations on networks.
\newblock {\em ACM SIGMETRICS Performance Evaluation Review} 42(1):289--300.

\bibitem[\protect\citeauthoryear{Caron \bgroup et al\mbox.\egroup
  }{2012}]{caron2012leveraging}
Caron, S.; Kveton, B.; Lelarge, M.; and Bhagat, S.
\newblock 2012.
\newblock Leveraging side observations in stochastic bandits.
\newblock In {\em Proceedings of the Twenty-Eighth Conference on Uncertainty in
  Artificial Intelligence (UAI)},  142--151.
\newblock AUAI Press.

\bibitem[\protect\citeauthoryear{Cesa-Bianchi and
  Lugosi}{2006}]{cesa2006prediction}
Cesa-Bianchi, N., and Lugosi, G.
\newblock 2006.
\newblock {\em Prediction, learning, and games}.
\newblock Cambridge university press.

\bibitem[\protect\citeauthoryear{Chen \bgroup et al\mbox.\egroup
  }{2016}]{chen2016combinatorial}
Chen, W.; Wang, Y.; Yuan, Y.; and Wang, Q.
\newblock 2016.
\newblock Combinatorial multi-armed bandit and its extension to
  probabilistically triggered arms.
\newblock {\em The Journal of Machine Learning Research (JMLR)}
  17(1):1746--1778.

\bibitem[\protect\citeauthoryear{Chen, Lakshmanan, and
  Castillo}{2013}]{chen2013information}
Chen, W.; Lakshmanan, L. V.~S.; and Castillo, C.
\newblock 2013.
\newblock {\em Information and Influence Propagation in Social Networks}.
\newblock Morgan \& Claypool Publishers.

\bibitem[\protect\citeauthoryear{Cohen, Hazan, and
  Koren}{2016}]{cohen2016online}
Cohen, A.; Hazan, T.; and Koren, T.
\newblock 2016.
\newblock Online learning with feedback graphs without the graphs.
\newblock In {\em International Conference on Machine Learning (ICML)},
  811--819.

\bibitem[\protect\citeauthoryear{Dontchev and
  Rockafellar}{2009}]{dontchev2009implicit}
Dontchev, A.~L., and Rockafellar, R.~T.
\newblock 2009.
\newblock Implicit functions and solution mappings.
\newblock {\em Springer Monogr. Math.}

\bibitem[\protect\citeauthoryear{Erd{\H{o}}s and
  R{\'e}nyi}{1960}]{erdHos1960evolution}
Erd{\H{o}}s, P., and R{\'e}nyi, A.
\newblock 1960.
\newblock On the evolution of random graphs.
\newblock {\em Publications of the Mathematical Institute of the Hungarian
  Academy of Sciences} 5:17--61.

\bibitem[\protect\citeauthoryear{Hoeffding}{1963}]{hoeffding1963probability}
Hoeffding, W.
\newblock 1963.
\newblock Probability inequalities for sums of bounded random variables.
\newblock {\em Journal of the American statistical association} 58(301):13--30.

\bibitem[\protect\citeauthoryear{Kempe, Kleinberg, and Tardos}{2003}]{kempe03}
Kempe, D.; Kleinberg, J.~M.; and Tardos, {\'E}.
\newblock 2003.
\newblock Maximizing the spread of influence through a social network.
\newblock In {\em Proceedings of the 9th ACM SIGKDD International Conference on
  Knowledge Discovery and Data Mining (KDD)},  137--146.

\bibitem[\protect\citeauthoryear{Koc{\'a}k \bgroup et al\mbox.\egroup
  }{2014}]{kocak2014efficient}
Koc{\'a}k, T.; Neu, G.; Valko, M.; and Munos, R.
\newblock 2014.
\newblock Efficient learning by implicit exploration in bandit problems with
  side observations.
\newblock In {\em Advances in Neural Information Processing Systems (NeurIPS)},
   613--621.

\bibitem[\protect\citeauthoryear{Koc{\'a}k, Neu, and
  Valko}{2016a}]{kocak2016online}
Koc{\'a}k, T.; Neu, G.; and Valko, M.
\newblock 2016a.
\newblock Online learning with erd{\H{o}}s-r{\'e}nyi side-observation graphs.
\newblock In {\em Uncertainty in Artificial Intelligence (UAI)}.

\bibitem[\protect\citeauthoryear{Koc{\'a}k, Neu, and
  Valko}{2016b}]{kocak2016onlinenoisy}
Koc{\'a}k, T.; Neu, G.; and Valko, M.
\newblock 2016b.
\newblock Online learning with noisy side observations.
\newblock In {\em Artificial Intelligence and Statistics (AISTATS)},
  1186--1194.

\bibitem[\protect\citeauthoryear{Komiyama, Honda, and
  Nakagawa}{2015}]{komiyama2015regret}
Komiyama, J.; Honda, J.; and Nakagawa, H.
\newblock 2015.
\newblock Regret lower bound and optimal algorithm in finite stochastic partial
  monitoring.
\newblock In {\em Advances in Neural Information Processing Systems (NeurIPS)},
   1792--1800.

\bibitem[\protect\citeauthoryear{Lattimore and
  Szepesvari}{2017}]{lattimore2017end}
Lattimore, T., and Szepesvari, C.
\newblock 2017.
\newblock The end of optimism? an asymptotic analysis of finite-armed linear
  bandits.
\newblock In {\em Artificial Intelligence and Statistics (AISTATS)},  728--737.

\bibitem[\protect\citeauthoryear{Lattimore and
  Szepesv{\'a}ri}{2019}]{lattimore2019cleaning}
Lattimore, T., and Szepesv{\'a}ri, C.
\newblock 2019.
\newblock Cleaning up the neighborhood: A full classification for adversarial
  partial monitoring.
\newblock In {\em Algorithmic Learning Theory (ALT)},  529--556.

\bibitem[\protect\citeauthoryear{Liu, Buccapatnam, and
  Shroff}{2018}]{liu2018information}
Liu, F.; Buccapatnam, S.; and Shroff, N.
\newblock 2018.
\newblock Information directed sampling for stochastic bandits with graph
  feedback.
\newblock In {\em Thirty-Second AAAI Conference on Artificial Intelligence
  (AAAI)}.

\bibitem[\protect\citeauthoryear{Mannor and Shamir}{2011}]{mannor2011bandits}
Mannor, S., and Shamir, O.
\newblock 2011.
\newblock From bandits to experts: On the value of side-observations.
\newblock In {\em Advances in Neural Information Processing Systems (NeurIPS)},
   684--692.

\bibitem[\protect\citeauthoryear{Rustichini}{1999}]{rustichini1999minimizing}
Rustichini, A.
\newblock 1999.
\newblock Minimizing regret: The general case.
\newblock {\em Games and Economic Behavior} 29(1-2):224--243.

\bibitem[\protect\citeauthoryear{Sarita{\c{c}} and
  Tekin}{2017}]{saritacc2017combinatorial}
Sarita{\c{c}}, A.~{\"O}., and Tekin, C.
\newblock 2017.
\newblock Combinatorial multi-armed bandit problem with probabilistically
  triggered arms: A case with bounded regret.
\newblock In {\em 2017 IEEE Global Conference on Signal and Information
  Processing (GlobalSIP)},  111--115.
\newblock IEEE.

\bibitem[\protect\citeauthoryear{Tossou, Dimitrakakis, and
  Dubhashi}{2017}]{tossou2017thompson}
Tossou, A.~C.; Dimitrakakis, C.; and Dubhashi, D.
\newblock 2017.
\newblock Thompson sampling for stochastic bandits with graph feedback.
\newblock In {\em Thirty-First AAAI Conference on Artificial Intelligence
  (AAAI)}.

\bibitem[\protect\citeauthoryear{Valiant}{1979}]{Valiant79}
Valiant, L.~G.
\newblock 1979.
\newblock The complexity of enumeration and reliability problems.
\newblock {\em SIAM Journal on Computing} 8(3):410--421.

\bibitem[\protect\citeauthoryear{Vaswani \bgroup et al\mbox.\egroup
  }{2015}]{vaswani2015influence}
Vaswani, S.; Lakshmanan, L.; Schmidt, M.; et~al.
\newblock 2015.
\newblock Influence maximization with bandits.
\newblock {\em arXiv preprint arXiv:1503.00024}.

\bibitem[\protect\citeauthoryear{Wang and Chen}{2017}]{wang2017improving}
Wang, Q., and Chen, W.
\newblock 2017.
\newblock Improving regret bounds for combinatorial semi-bandits with
  probabilistically triggered arms and its applications.
\newblock In {\em Advances in Neural Information Processing Systems (NeurIPS)},
   1161--1171.

\bibitem[\protect\citeauthoryear{Wang, Chen, and Wang}{2012}]{wang2012scalable}
Wang, C.; Chen, W.; and Wang, Y.
\newblock 2012.
\newblock Scalable influence maximization for independent cascade model in
  large-scale social networks.
\newblock {\em Data Mining and Knowledge Discovery} 25(3):545--576.

\bibitem[\protect\citeauthoryear{Wen \bgroup et al\mbox.\egroup
  }{2017}]{wen2017online}
Wen, Z.; Kveton, B.; Valko, M.; and Vaswani, S.
\newblock 2017.
\newblock Online influence maximization under independent cascade model with
  semi-bandit feedback.
\newblock In {\em Advances in Neural Information Processing Systems (NeurIPS)},
   3025--3035.

\bibitem[\protect\citeauthoryear{Wu, Gy{\"o}rgy, and
  Szepesv{\'a}ri}{2015}]{wu2015online}
Wu, Y.; Gy{\"o}rgy, A.; and Szepesv{\'a}ri, C.
\newblock 2015.
\newblock Online learning with gaussian payoffs and side observations.
\newblock In {\em Advances in Neural Information Processing Systems (NeurIPS)},
   1360--1368.

\end{thebibliography}

\ifsup
\newpage
\onecolumn
\appendix

\section{Proofs of the Upper Bounds in One-Step Triggering}
\label{sec:proofs of upper bound one step}

\begin{proof}[of Theorem \ref{thm:upper bound:uniform case}]

	Define events
	\begin{align*}
	\cA_t &=\set{M_j(t) < 10 \log\left( K t^2 \right), \quad \text{ for some } j \in [K] }\\
	\cB_t &=\set{m_j(t) < M_j(t) /2, \quad \text{ for some } j \in [K] } \\
	\cC_t &= \set{ \abs{\hat{\theta}_{t, i} - \theta_i} \ge \sqrt{\frac{2 \log(t)}{m_i(t)}}, \quad \forall i \in [K] }\\
	\cD_t &= \set{ \frac{N(t)}{16\log(t)} \in C(\hat{\theta}_{t}) }\\
	\cE_t &= \set{ M_j(t) < 2 \ \beta(N^e(t)) \ / \ K, \quad \text{ for some } j \in [K] }\\
	\cF_t &= \set{ \abs{\hat{\theta}_{t,i} - \theta_i} \le \epsilon, \quad \text{ for any } i \in [K] }
	\end{align*}

	\paragraph{Bound the regret under $\cB$}

	Note
	\begin{align}
	\sum_{t=1}^T \EE{\Delta_{i_t}(\theta) \bOne{\cB_{t-1}}} &\le \sum_{t=1}^T \EE{\Delta_{i_t}(\theta) \bOne{\cB_{t-1}, \cA_{t-1}^c}} + \sum_{t=1}^T \EE{\Delta_{i_t}(\theta) \bOne{\cB_{t-1}, \cA_{t-1}}} \notag \\
	&\le \Delta_{\max}(\theta) \pi^2 / 6 + \sum_{t=1}^T \EE{\Delta_{i_t}(\theta) \bOne{N_{i_t}(t-1) < \frac{10}{p} \log(Kt^2)}} \notag \\
	&\le \Delta_{\max}(\theta) \pi^2 / 6 + \sum_{i=1}^K \frac{10 \Delta_i(\theta)}{p} \log(KT^2) \label{eq:regret5}
	\end{align}
	where the first term is by Lemma \ref{lem:half bernoulli}.

	\paragraph{Bound the regret under $\cC$}
	\begin{align}
	\sum_{t=1}^T \EE{\Delta_{i_t}(\theta) \bOne{\cC_{t-1}}} \le 2K\Delta_{\max}(\theta) \label{eq:regret2}
	\end{align}

	Then it remains to bound $\sum_{t=1}^T \EE{\Delta_{i_t}(\theta) \bOne{\cB_{t-1}^c, \cC_{t-1}^c}}$.

	\paragraph{Bound the regret under $\cD$}

	Suppose $\cD_{t-1}$ and $\cB_{t-1}^c$, $\cC_{t-1}^c$ hold. Then $\sum_{i \in V^{\iin}(j)} N_i(t-1) p \ge \frac{32}{\Delta_j(\hat{\theta}_{t-1})^2} \log(t-1)$ for any $j \ne i_1(\hat{\theta}_{t-1})$ and $\sum_{i \in V^{\iin}(j)} N_i(t-1) p \ge \frac{32}{\Delta_{i_2(\hat{\theta}_{t-1})}(\hat{\theta}_{t-1})^2} \log(t-1)$ for $j = i_1(\hat{\theta}_{t-1})$. Or equivalently
	\begin{align*}
	&M_j(t-1) \ge \frac{32}{\Delta_j(\hat{\theta}_{t-1})^2} \log(t-1) \quad \text{ for } j \ne i_1(\hat{\theta}_{t-1})\,,\\ 
	&M_j(t-1) \ge \frac{32}{\Delta_{i_2(\hat{\theta}_{t-1})}^2(\hat{\theta}_{t-1})} \log(t-1) \quad \text{ for } j = i_1(\hat{\theta}_{t-1})\,.
	\end{align*}
	On $\cB_{t-1}^c$,
	\begin{align*}
	&m_j(t-1) \ge \frac{16}{\Delta_j(\hat{\theta}_{t-1})^2} \log(t-1) \quad \text{ for } j \ne i_1(\hat{\theta}_{t-1})\,,\\ 
	&m_j(t-1) \ge \frac{16}{\Delta_{i_2(\hat{\theta}_{t-1})}^2(\hat{\theta}_{t-1})} \log(t-1) \quad \text{ for } j = i_1(\hat{\theta}_{t-1})\,.
	\end{align*}
	Then
	\begin{align}
	\sum_{t=1}^T \EE{\Delta_{i_t}(\theta) \bOne{\cB_{t-1}^c, \cC_{t-1}^c}, \cD_{t-1}} = 0 \label{eq:regret3}
	\end{align}
	since
	\begin{align*}
	\theta_{i_1(\hat{\theta}_{t-1})} \ge \hat{\theta}_{t-1,i_1(\hat{\theta}_t)} - \sqrt{\frac{2\log(t-1)}{m_{i_1(\hat{\theta}_{t-1})}(t-1)}} &\ge \hat{\theta}_{t-1,i_1(\hat{\theta}_{t-1})} - \frac{\Delta_{i_2(\hat{\theta}_{t-1})}(\hat{\theta}_{t-1})}{2}\\
	&\ge \hat{\theta}_{t-1, i} + \frac{\Delta_{i}(\hat{\theta}_{t-1})}{2} \ge \theta_i
	\end{align*}
	thus $i_t = i_1(\hat{\theta}_{t-1}) = i_1(\theta)$. 

	Thus it remains to bound $\sum_{t=1}^T \EE{\Delta_{i_t}(\theta) \bOne{\cB_{t-1}^c, \cC_{t-1}^c, \cD_{t-1}^c}}$.

	\paragraph{Bound the regret under $\cB^c, \cC^c, \cD^c$}

	Similar to \cite[Proposition 17]{wu2015online} where the statement $\sum_{i \in V^\iin(j)} N_{ij} \ge \beta(s) / K$ is replaced by $M_j(t) \ge \beta(s) / K$,
	\begin{align*}
	\sum_{t=K}^T \bOne{\cB_{t-1}^c, \cD_{t-1}^c, \cE_{t-1}} \le 1 + \beta\left(\sum_{t=K+1}^T \bOne{\cB_{t-1}^c, \cD_{t-1}^c}\right)\,.
	\end{align*}
	Then
	\begin{equation}
	\label{eq:aux1}
	\begin{split}
	&\sum_{t=K+1}^T \bOne{\cB_{t-1}^c, \cC_{t-1}^c, \cD_{t-1}^c, \cE_{t-1}} \\
	&\le 2 + \sum_{t=K+1}^T\bOne{\cC_{t-1}} + \sum_{t=K+1}^T \bOne{\cB_{t-1}^c, \cC_{t-1}^c, \cD_{t-1}^c, \cE_{t-1}^c, \cF_{t-1}^c} + 2\beta\left(\sum_{t=K+1}^n \bOne{\cB_{t-1}^c, \cC_{t-1}^c, \cD_{t-1}^c, \cE_{t-1}^c, \cF_{t-1}}\right)\,.
	\end{split}
	\end{equation}

	Next by \cite[Lemma 19]{wu2015online},
	\begin{align}
	&\sum_{t=1}^T \EE{\Delta_{i_t}(\theta) \bOne{\cB_{t-1}^c, \cC_{t-1}^c, \cD_{t-1}^c, \cE_{t-1}^c, \cF_{t-1}^c}} \le \sum_{s=0}^T 2\exp\left(-\frac{\beta(s)\epsilon^2}{2K}\right)\,, \label{eq:regret6}\\
	&\sum_{t=1}^T \bOne{\cB_{t-1}^c, \cC_{t-1}^c, \cD_{t-1}^c, \cE_{t-1}^c, \cF_{t-1}} \le K + 4\sum_{i=1}^K c_i(\theta, \epsilon) \log(T) \label{eq:aux2}\,, \\
	&\sum_{t=1}^T \EE{\Delta_{i_t}(\theta) \bOne{\cB_{t-1}^c, \cC_{t-1}^c, \cD_{t-1}^c, \cE_{t-1}^c, \cF_{t-1}}} \le K + 4\sum_{i=1}^K c_i(\theta, \epsilon) \Delta_i(\theta) \log(T)\,. \label{eq:regret7}
	\end{align}

	Thus by \eqref{eq:aux1}, \eqref{eq:regret2}, \eqref{eq:regret6} and \eqref{eq:aux2},
	\begin{align}
	&\sum_{t=1}^T \EE{\Delta_{i_t}(\theta)\bOne{\cB_{t-1}^c, \cD_{t-1}^c, \cE_{t-1}}} \notag\\
	\le &\sum_{t=1}^T \EE{\Delta_{i_t}(\theta) \bOne{\cC_{t-1}}} + K\Delta_{\max}(\theta) + \sum_{t=K+1}^T \EE{\Delta_{i_t}(\theta)\bOne{\cB_{t-1}^c, \cC_{t-1}^c, \cD_{t-1}^c, \cE_{t-1}}} \notag\\
	\le &2K\Delta_{\max}(\theta) + K\Delta_{\max}(\theta) + 2 + 2K\Delta_{\max}(\theta) + \sum_{s=0}^T 2\exp\left(-\frac{\beta(s)\epsilon^2}{2K}\right) + 2\beta(1 + 4 c_i(\theta, \epsilon) \log(T)) \notag\\
	\le &2 + 5K\Delta_{\max}(\theta) + 2 \sum_{s=0}^T \exp\left(-\frac{\beta(s)\epsilon^2}{2K}\right) + 2\beta\left(K + 4 \sum_{i=1}^K c_i(\theta, \epsilon) \log(T)\right)\,. \label{eq:regret8}
	\end{align}

	Putting \eqref{eq:regret5}, \eqref{eq:regret2}, \eqref{eq:regret3}, \eqref{eq:regret6}, \eqref{eq:regret7}, \eqref{eq:regret8} together, the regret satisfies
	\begin{align*}
	R_{\theta}(T) \le &3 + K + \left(7K + \frac{\pi^2}{3}\right) \Delta_{\max}(\theta) + 4 \sum_{s=0}^T \exp\left(-\frac{\beta(s)\epsilon^2}{2K}\right) + 10 \sum_{i=1}^K \frac{\Delta_i(\theta)}{p} \log(KT^2) \\
	&+ 2\beta\left(K + 4 \sum_{i=1}^K c_i(\theta, \epsilon) \log(T)\right) + 4\sum_{i=1}^K c_i(\theta, \epsilon) \Delta_i(\theta) \log(T)\,.
	\end{align*}

	Next prove the asymptotic behavior of the regret upper bound. 
	\begin{claim}
	$M_j(t) \to \infty$ as $t \to \infty$ for any $j \in [K]$.
	\end{claim}

	Suppose not. There exists $j \in [K]$ such that $M_j(t)$, or $N_i(t)$ for all $i \in V^\iin(j)$, stops increasing when $t \ge T_1$ for some $T_1 > 0$. Then the condition on line \ref{alg:uniform:if good condition} is not satisfied when $t \ge T_2 \ge T_1$ for some $T_2>0$. By the condition on line \ref{alg:uniform:if accurate theta}, $N^e(t)$ also stops increasing and the condition on line \ref{alg:uniform:if accurate theta} for any $j'$ is not satisfied any more when $t \ge T_3 \ge T_2$ for some $T_3>0$. Also line \ref{alg:uniform:explore C, increase Ne} will not be performed since $N^e(t)$ stops increasing. Therefore the condition on line \ref{alg:uniform:if explore p} always holds, which is impossible.

	For any $\delta \in (0,1)$, the probability that the condition on line \ref{alg:uniform:if explore p} does not hold when $M_j(t) > 10 \log\frac{K}{\delta}$ is at least $1-\delta/K$. There exists $T_4>0$ such that when $t \ge T_4$, $M_j(t)>10 \log\frac{K}{\delta}$ for any $j$ since $M_j(t) \to \infty$. Then with probability at least $1-\delta$, line \ref{alg:uniform:explore p}-\ref{alg:uniform:explore p, not increase Ne} are not called any more. The events $\cA_t$ is modified by $\cA_t' = \set{M_j(t) < 10 \log\frac{K}{\delta} \text{ for some } j \in [K]}$ and \eqref{eq:regret5} is replaced by $\sum_{t=1}^T \EE{\Delta_{i_t}(\theta) \bOne{N_{i_t}(t-1) < \frac{10}{p} \log\frac{K}{\delta}}} \le \sum_{i=1}^K \frac{10 \Delta_i(\theta)}{p} \log(K/\delta)$. All other parts stay the same. Then the regret satisfies
	\begin{align*}
	R_{\theta}(T) \le &3 + K + \left(7K + \frac{\pi^2}{3}\right) \Delta_{\max}(\theta) + 4 \sum_{s=0}^T \exp\left(-\frac{\beta(s)\epsilon^2}{2K}\right) + 10 \sum_{i=1}^K \frac{\Delta_i(\theta)}{p} \log(K/\delta) \\
	&+ 2\beta\left(K + 4 \sum_{i=1}^K c_i(\theta, \epsilon) \log(T)\right) + 4\sum_{i=1}^K c_i(\theta, \epsilon) \Delta_i(\theta) \log(T)\,.
	\end{align*}

	For any $\eta>0$, there exists an $\epsilon = \epsilon(\theta) > 0$ such that the distance between the optimal solution set of $c(\theta)$ and $c(\theta')$ for any $\theta'$ such that $\abs{\theta_i'-\theta_i} \le \epsilon$ for all $i \in [K]$ is at most $\eta$. Here the distance is Pompeiu-Hausdorff distance of sets. This is because the Lipschitz continuity of the optimal set mapping (see \cite[\S3C.5]{dontchev2009implicit}) and the duality of linear programming problems. Since $c(\theta)$ is unique, $c_i(\theta, \epsilon)$ is upper bounded by $c_i(\theta) + \eta$. Then divide $R_{\theta}(T)$ by $\log(T)$ and let $T$ go to $\infty$,
	\begin{align}
	\limsup_{T \to \infty} \frac{R_\theta(T)}{\log(T)} \le 4 \inf_{c \in C(\theta)} \ip{ c, \Delta(\theta) } + \eta \sum_{i=1}^K \Delta_i(\theta)\,.
	\end{align}
	For instance, $\eta$ can be chosen as $\ip{c(\theta), \Delta(\theta)} / \sum_{i=1}^K \Delta_i(\theta)$.

\end{proof}

The proof of Theorem \ref{thm:upper bound:general case} is similar to Theorem \ref{thm:upper bound:uniform case} by modifying \eqref{eq:regret5} with
\begin{align}
\sum_{t=1}^T \EE{\Delta_{i_t}(\theta) \bOne{\cB_{t-1}}} &\le \sum_{t=1}^T \EE{\Delta_{i_t}(\theta) \bOne{\cB_{t-1}, \cA_{t-1}^c}} + \sum_{t=1}^T \EE{\Delta_{i_t}(\theta) \bOne{\cB_{t-1}, \cA_{t-1}}} \notag \\
&\le \Delta_{\max}(\theta) \pi^2 / 6 + \sum_{t=1}^T \EE{\Delta_{i_t}(\theta) \bOne{i_t \in V^e, N_{i_t}(t-1) < \frac{10}{p_{i_t}^e} \log(Kt^2)}} \notag \\
&\le \Delta_{\max}(\theta) \pi^2 / 6 + \sum_{i \in V^e} \frac{10 \Delta_i(\theta)}{p_{i}^e} \log(KT^2) \,.
\end{align}


\section{Proof of Theorem \ref{thm:upper bound:cascade case}}
\label{sec:proof of cascade case}

First we prove a useful lemma on the robustness of linear programming problem where the coefficient matrix is of a specific form.

\begin{lemma}
	\label{lem:noisy lp}
	Denote the linear programming problem of the form
	\begin{align}
	    \text{ minimize } \ip{c, x} \text{ over all } x \in \RR^n \text{ satisfying } Ax \ge b \text{ and } x \ge 0
	\end{align}
	by $\LP(A,b,c)$ where $A \in \RR^{n \times n}, b \in \RR^n, c \in \RR^n$ and all entries in $A,b,c$ are non-negative. Let the feasible set mapping, the optimal value mapping and the optimal set mapping be 
	\begin{align*}
	&S_{\feas}(A;b) = \set{x \mid Ax \ge b, x \ge 0}\\
	&S_{\val}(A;b,c) = \inf_x \set{\ip{c,x} \mid x \in S_{\feas}(A,b)}\\
	&S_{\opt}(A;b,c) = \set{x \in S_{\feas}(A,b) \mid \ip{c,x} = S_{\val}(A,b,c)}
	\end{align*}
	respectively. 
	Note that $S_{\val}(A,b,c)$ is always finite with the $A,b,c$ of positive (or even non-negative) entries. 

	Fix a pair $(i,j)$, assume $A(i,j) > 0$. Let $A' = A$ except $A'(i,j) = A(i,j) + \epsilon$. Then
	\begin{align}
	d_H(S_{\opt}(A), S_{\opt}(A')) \le \alpha \abs{\epsilon}
	\end{align}
	for some $\alpha$ depending on $A,b,c$ and $i,j$.
\end{lemma}
\begin{proof}
	By \cite[Theorem 3C.3]{dontchev2009implicit}, the mapping
	\begin{align}
	\label{eq:G mapping}
	G: t \mapsto \set{x \in S_{\feas}(A, b) \mid \ip{c,x} \le t}
	\end{align}
	is Lipschitz continuous. Recall that the distance on sets is Pompeiu-Hausdorff distance $d_H$.

	First assume $\epsilon>0$. Then
	\begin{align*}
	S_{\feas}(A) \subset S_{\feas}(A'),\quad S_{\val}(A) \ge S_{\val}(A')\,.
	\end{align*}
	If $S_{\opt}(A') \subset S_{\feas}(A)$, then $S_{\val}(A) = S_{\val}(A')$. Suppose not and let $x' \in S_{\opt}(A') \setminus S_{\feas}(A)$.
	Then
	\begin{align*}
	0 < b - \sum_{j'=1}^n A(i,j') x_{j'}' \le \epsilon x_j'\,.
	\end{align*}
	Let $x = x'$ except $x_j = x_j' + \frac{\epsilon}{A(i,j)} x_j'$, then $x \in S_{\feas}(A)$ and
	\begin{align*}
	S_{\val}(A) \le \ip{c, x} = \ip{c, x'} + \frac{\epsilon}{A(i,j)} c_j x_j' \le S_{\val}(A') + \alpha_1 \epsilon
	\end{align*}
	for some $\alpha_1$ depending on $A,b,c$ and $i,j$. Thus
	\begin{align*}
	S_{\opt}(A') \subset G(S_{\val}(A') + \alpha_1 \epsilon) + \alpha_2 \epsilon B_1 \subset G(S_{\val}(A)) + \alpha_3 \epsilon B_1 = S_{\opt}(A) + \alpha_3 \epsilon B_1
	\end{align*}
	for some $\alpha_2, \alpha_3$ depending on $A,b,c$,
	where $B_1$ is the unit ball in $\RR^n$ and the second inequality is due to the Lipshitz continuity of $G$ in \eqref{eq:G mapping}. Also
	\begin{align*}
	S_{\opt}(A) \subset S_{\feas}(A) \subset S_{\feas}(A') \subset G'(S_{\val}(A)) \overset{(*)}{\subset} G'(S_{\val}(A')) + \beta_4 \epsilon B_1 = S_{\opt}(A') + \beta_4 \epsilon B_1
	\end{align*}
	where (*) is by by the Lipshitz continuity of $G':t \mapsto \set{x \in S_{\feas}(A', b) \mid \ip{c,x} \le t}$.

	The case of $\epsilon < 0$ follows similarly.
\end{proof}

\begin{proof}[of Theorem \ref{thm:upper bound:cascade case}]

	The finite-time regret is similar to the previous proof. The main difference is on the bound for line \ref{alg:cascade:explore C}. In particular, the results of \eqref{eq:aux2} and \eqref{eq:regret7} are changed to be
	\begin{align*}
	&\sum_{t=1}^T \bOne{\cB_{t-1}^c, \cC_{t-1}^c, \cD_{t-1}^c, \cE_{t-1}^c, \cF_{t-1}} \le K + 4 \sum_{i=1}^K \max_{t \in [T]} \set{c_i(\theta, \epsilon, \eta(t)) \log(t)}\,, \\
	&\sum_{t=1}^T \EE{\Delta_{i_t}(\theta) \bOne{\cB_{t-1}^c, \cC_{t-1}^c, \cD_{t-1}^c, \cE_{t-1}^c, \cF_{t-1}}} \le K + 4 \sum_{i=1}^K \Delta_i(\theta) \max_{t \in [T]} \set{c_i(\theta, \epsilon, \eta(t)) \log(t)} 
	\end{align*}
	since $c_i(\theta, \epsilon, \eta(t))$ can bound $S_{\opt}(\LP_t)$.
	The proof of other parts follow the proof of \ref{thm:upper bound:general case} similarly.

	By the non-increasing property of $\eta(t)$ whose limit is $0$, $\eta(t)$ would be smaller than $\min_{i j: p_{ij}'>0} p_{ij}'$ when $t \ge T_1$ for some $T_1>0$. Then $P_t$ only has small noise on the nonzero entries of $P' = (p_{ij}')_{ij}$. By Lemma \ref{lem:noisy lp}, $S_{\opt}(\theta', \eta(t))$ is Lipschitz continuous in $\eta(t)$ for any $\theta'$. Thus
	\begin{align*}
	\lim_{t \to \infty} c_i(\theta, \epsilon, \eta(t)) = c_i(\theta, \epsilon)\,.
	\end{align*}
	The remaining discussion on $\epsilon$ is similar.
\end{proof}

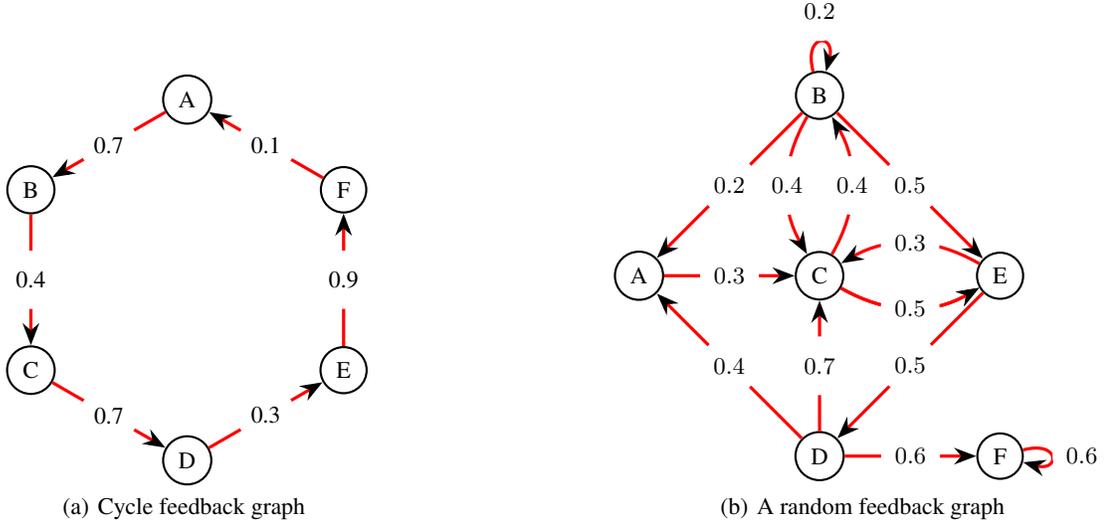
\begin{figure}[t]
		\label{fig:graphs}
		\centering

		\subfigure[Cycle feedback graph]{
			\label{fig:graphs a}
    		\begin{minipage}[t]{0.5\linewidth}
        	\centering
        	\begin{tikzpicture}[scale=0.8,font=\small]
				\def\r{3}
				\def\c{7}
				\begin{scope}[every node/.style={circle,thick,draw,, minimum size=0.5cm}]
				\foreach \x/\y in {90/A,150/B,210/C,270/D,330/E,30/F}{
					\node (\y) at ({\c+\r*cos(\x)},{\c+\r*sin(\x)}) {\y};
				}
				\end{scope}
				\begin{scope}[>={Stealth[black]},
				every node/.style={fill=white,circle},
				every edge/.style={draw=red,very thick}]
				\path [->] (A) edge node {0.7} (B)
				(B) edge node {0.4} (C)
				(C) edge node {0.7} (D)
				(D) edge node {0.3} (E)
				(E) edge node {0.9} (F)
				(F) edge node {0.1} (A);
				\end{scope}
			\end{tikzpicture}
    		\end{minipage}%
		}%
		\subfigure[A random feedback graph]{
			\label{fig:graphs b}
    		\begin{minipage}[t]{0.5\linewidth}
        	\centering
        	\begin{tikzpicture}[scale=0.8,font=\small]
				\def\r{3}
				\def\c{4}
				\begin{scope}[every node/.style={circle,thick,draw, minimum size=0.5cm}]
				\node (A) at (\c-\r,\c) {A};
				\node (B) at (\c,\c+\r) {B};
				\node (C) at (\c,\c) {C};
				\node (D) at (\c,\c-\r) {D};
				\node (E) at (\c+\r,\c) {E};
				\node (F) at (\c+\r,\c-\r) {F} ;
				\end{scope}

				\begin{scope}[>={Stealth[black]},
				every node/.style={fill=white,circle},
				every edge/.style={draw=red,very thick}]
				\path [->] (A) edge node {$0.3$} (C);
				\path [->] (B) edge node {$0.2$} (A);
				\path [->,draw=red,very thick] (B) to[loop above] node {$0.2$} (B);
				\path [->,draw=red,very thick] (B) to[bend right] node {$0.4$} (C);
				\path [->] (B) edge node {$0.5$} (E);
				\path [->,draw=red,very thick] (C) to[bend right] node {$0.4$} (B);
				\path[->,draw=red,very thick] (C) to[bend right] node {$0.5$} (E);
				\path [->] (D) edge node {$0.6$} (F);
				\path [->] (D) edge node {$0.4$} (A);
				\path [->] (D) edge node {$0.7$} (C);
				\path [->,draw=red,very thick] (E) to[bend right] node {$0.3$} (C); 
				\path [->] (E) edge node {$0.5$} (D); 
				\path [->,draw=red,very thick] (F) to[loop right] node {$0.6$} (F); 
				\path[->,draw=red,very thick] (C) to[bend right] node {$0.5$} (E);

				\end{scope}
			\end{tikzpicture}
    		\end{minipage}
		}%
		\caption{Two feedback graphs on $6$ nodes}
	\end{figure}

	\begin{figure}[t]
		\centering
  		\includegraphics[width=0.45\textwidth]{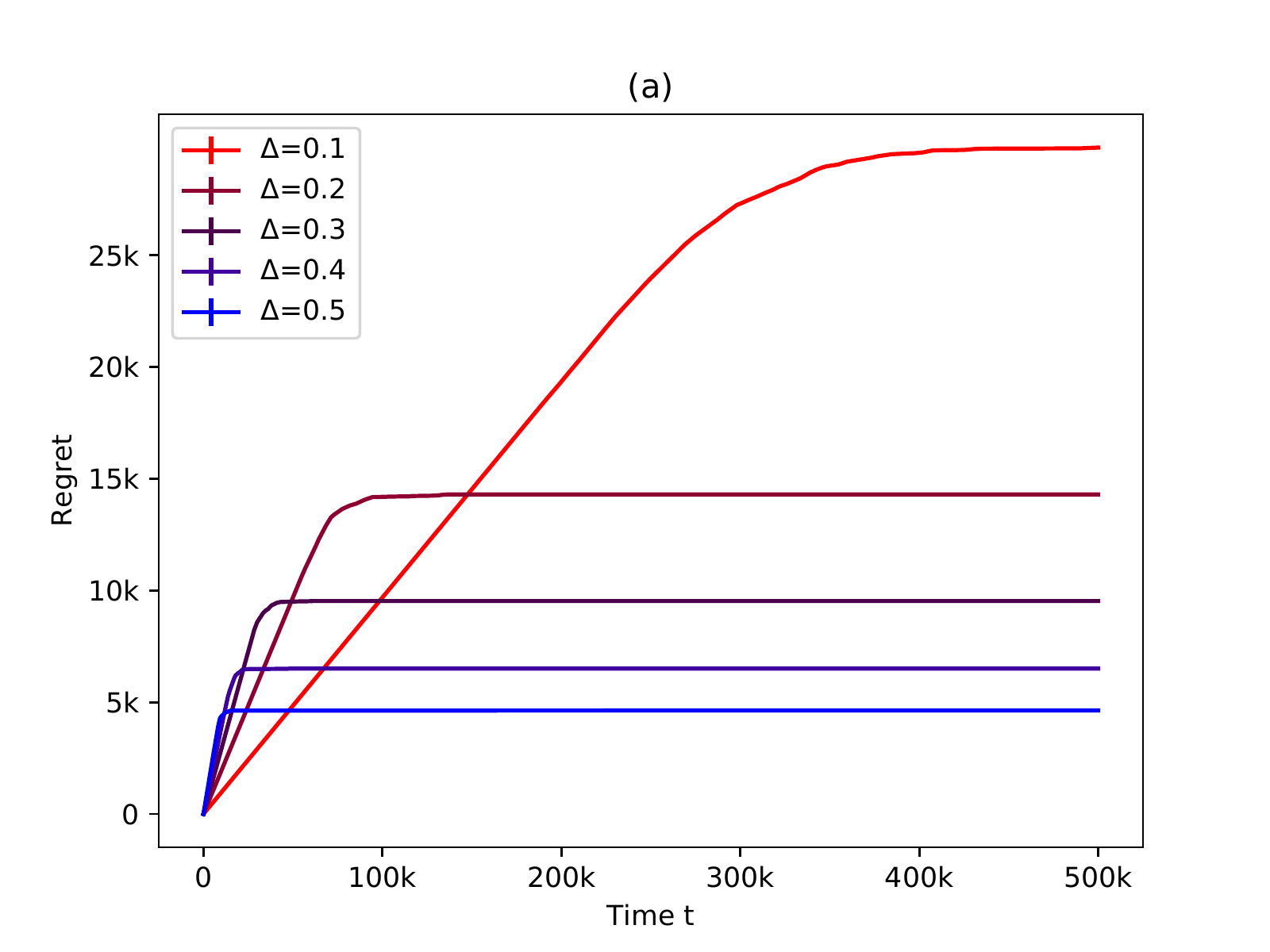}
  		\includegraphics[width=0.45\textwidth]{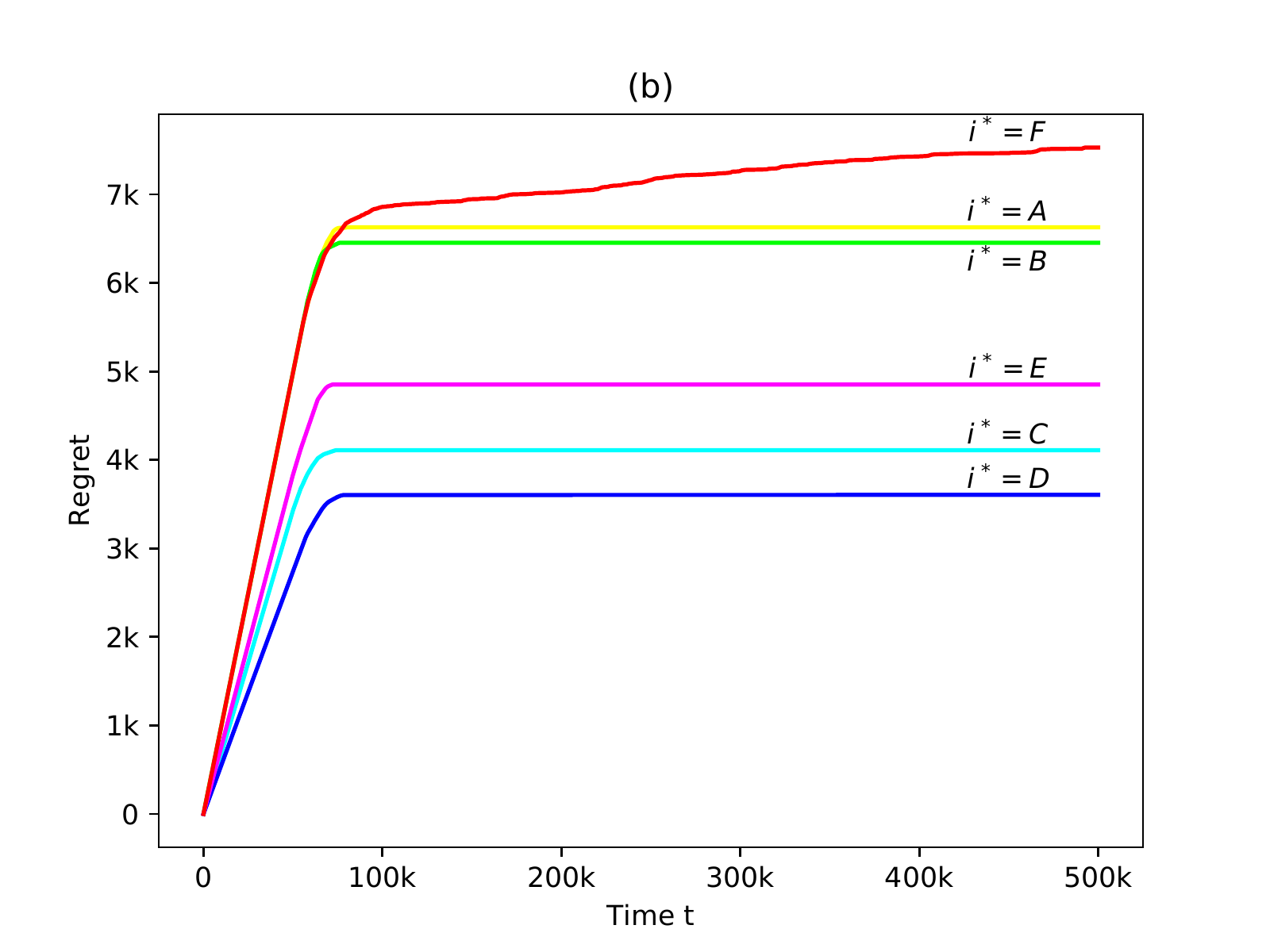}
 		\caption{Regrets on two feedback graphs}
 		\label{fig:results}
	\end{figure}

\section{Technical Lemmas}

\begin{lemma}[Hoeffding's Inequality \cite{hoeffding1963probability}]
\label{lem:chernoff bound}
Let $X_1, \ldots, X_n$ be independent random variable with common support $[0, 1]$. Let $\bar{X} = \frac{1}{n} \sum_{i=1}^n X_i$ and $\EE{\bar{X}} = \mu$. Then for all $a \ge 0$,
\begin{align*}
\PP{\bar{X} - \mu \ge a} \le  \exp(- 2 n a^2),\quad \PP{\bar{X} - \mu \le -a} \le  \exp(- 2 n a^2).
\end{align*}
\end{lemma}

\begin{lemma}[Bernstein's Inequality]
\label{lem:bernstein ineq}
Let $X_1, \ldots, X_n$ be independent zero-mean random variables. Suppose that $\abs{X_i} \le M$ almost surely for all $i$. Then for all $a \ge 0$,
\begin{align*}
\PP{\sum_{i=1}^n X_i \ge a} \le \exp\left(-\frac{a^2/2}{\sum_{i=1}^n \EE{X_i^2} + Ma/3} \right)\,.
\end{align*}
\end{lemma}

\begin{lemma} 
\label{lem:half bernoulli}
Let $x_1, x_2, \ldots, x_t$ be independent Bernoulli random variables with mean $p_1, p_2, \ldots, p_t \in (0,1)$ respectively.
\begin{align*}
\PP{\sum_{s=1}^t x_s < \frac{1}{2}\sum_{s=1}^t p_s} \le \delta
\end{align*}
if $\sum_{s=1}^t p_s \ge 10 \log\left(\frac{1}{\delta}\right)$.
\end{lemma}
\begin{proof}
\begin{align}
\PP{\sum_{s=1}^t x_s < \frac{1}{2}\sum_{s=1}^t p_i} &= \PP{\sum_{s=1}^t p_s - \sum_{s=1}^t x_s > \frac{1}{2}\sum_{s=1}^t p_s} \notag \\
&\le \exp \left( - \frac{(\sum_{s=1}^t p_s)^2/8}{\sum_{s=1}^t p_s(1-p_s) + (\sum_{s=1}^t p_s)/6} \right) \label{eq:cite bernstein}\\
&\le \exp \left( - \frac{(\sum_{s=1}^t p_s)^2/8}{\sum_{s=1}^t p_s + (\sum_{s=1}^t p_s)/6} \right)\notag\\
&\le \exp \left( - \frac{(\sum_{s=1}^t p_s)/8}{7/6} \right)\notag\\
&\le \delta \label{eq:half bernoulli last}
\end{align}
where \eqref{eq:cite bernstein} is by Bernstein's inequality (Lemma \ref{lem:bernstein ineq}) and \eqref{eq:half bernoulli last} holds when
\begin{align*}
\sum_{s=1}^t p_s \ge 10 \log\left(\frac{1}{\delta}\right)\,.
\end{align*}
\end{proof}

\section{Experiments}
	
	This section demonstrates two simple experiments for the cascade case with $6$ nodes and the reward random variables are Gaussian with unit variance. The first uses a cycle graph (see Figure \ref{fig:graphs a}) where the probabilities on the edges are generated randomly. We set the reward mean vector for the $6$ nodes $A,B,\ldots, F$ to be $\theta=(0.5+\Delta,0.5,\ldots,0.5)$. We run our Algorithm \ref{algo:cascade case} with different $\Delta$'s. The results are shown in Figure \ref{fig:results}(a) and each regret curve is averaged over $10$ random runs.

	The second uses a random graph (see Figure \ref{fig:graphs b}) where both the edges and the probabilities on the edges are generated randomly. We test our Algorithm \ref{algo:cascade case} for $6$ cases, each selects a best arm $i^\ast=O$ ($O$ can be $A,B,\ldots,F$), where the reward mean for $O$ is $0.6$ and the reward mean for others is $0.5$. The regret results are shown in Figure \ref{fig:results}(b) with each taking average of $10$ random runs.
\fi

\end{document}